\documentclass[11pt,a4paper]{article}
\usepackage[hyperref]{emnlp2020}
\usepackage{times}
\usepackage{latexsym}

\usepackage{microtype}

\aclfinalcopy 

\usepackage{appendix}
\usepackage{graphicx}
\usepackage{bm}
\usepackage{amssymb}
\usepackage{amsmath}
\usepackage{booktabs}
\usepackage{xcolor}
\usepackage{multirow}
\usepackage{booktabs}
\usepackage{verbatim}
\usepackage{mdwlist}
\usepackage[ruled,norelsize,linesnumbered]{algorithm2e}
\usepackage{amsthm}
\usepackage[capitalize]{cleveref}
\usepackage{paralist}
\usepackage{enumitem}
\usepackage{subcaption}

\theoremstyle{definition}
\newtheorem{definition}{Definition}[section]
\newtheorem{theorem}{Theorem}
\newtheorem{prop}{Proposition}

\DeclareMathOperator*{\softmax}{softmax}

\newcommand{\e}[2]{\mathbb{E}_{#1}\left[ #2 \right] }
\newcommand{\h}[2]{\mathbb{H}_{#1}\left[ #2 \right] }

\newcommand{\abr}{\textsc}
\newcommand{\kld}[2]{D_{\mathrm{KL}} \left[ \left. \left. #1 \right|\right| #2 \right] }
\newcommand{\mi}[1]{I_{#1}}

\makeatletter
\newcommand{\removelatexerror}{\let\@latex@error\@gobble}
\makeatother

\title{Dual Reconstruction: a Unifying Objective \\ for Semi-Supervised Neural Machine Translation}

\author{Weijia Xu \\
	University of Maryland \\
	{\tt \href{mailto:weijia@cs.umd.edu}{weijia@cs.umd.edu}} \\\And
	Xing Niu\Thanks{\ Work was done at the University of Maryland.} \\
	Amazon AI \\
	{\tt \href{mailto:xingniu@amazon.com}{xingniu@amazon.com}} \\\And
	Marine Carpuat \\
	University of Maryland \\
	{\tt \href{mailto:marine@cs.umd.edu}{marine@cs.umd.edu}} \\}

\date{}

\begin{document}
\maketitle
\begin{abstract}
While Iterative Back-Translation and Dual Learning effectively incorporate monolingual training data in neural machine translation, they use different objectives and heuristic gradient approximation strategies, and have not been extensively compared. We introduce a novel dual reconstruction objective that provides a unified view of Iterative Back-Translation and Dual Learning. It motivates a theoretical analysis and controlled empirical study on German-English and Turkish-English tasks, which both suggest that Iterative Back-Translation is more effective than Dual Learning despite its relative simplicity.
\end{abstract}

\section{Introduction}

Taking advantage of monolingual training data via Back-Translation~\citep{SennrichHB16}, Iterative Back-Translation~\citep{ZhangLZC18,CotterellK18} or Dual Learning~\citep{HeXQWYLM16}  has become a de facto requirement for building high quality Neural Machine Translation (NMT) systems~\citep{EdunovOAG18,Hassan18}. 
However, these methods rely on unrelated heuristic optimization objectives, and it is not clear what their respective strengths and weaknesses are, nor how they relate to the ideal but intractable objective of maximizing the marginal likelihood of the monolingual data (i.e.,~$p_\theta(\boldsymbol{y}) = \sum_{\boldsymbol{x}} p_\theta(\boldsymbol{y}\, | \,\boldsymbol{x}) q(\boldsymbol{x})$ given target sentences~$\boldsymbol{y}$, an NMT  model~$p_\theta(\boldsymbol{y}\, | \,\boldsymbol{x})$, and the prior distribution~$q(\boldsymbol{x})$ on source~$\boldsymbol{x}$).

\begin{figure}[ht]
    \centering
    \includegraphics[width=0.48\textwidth]{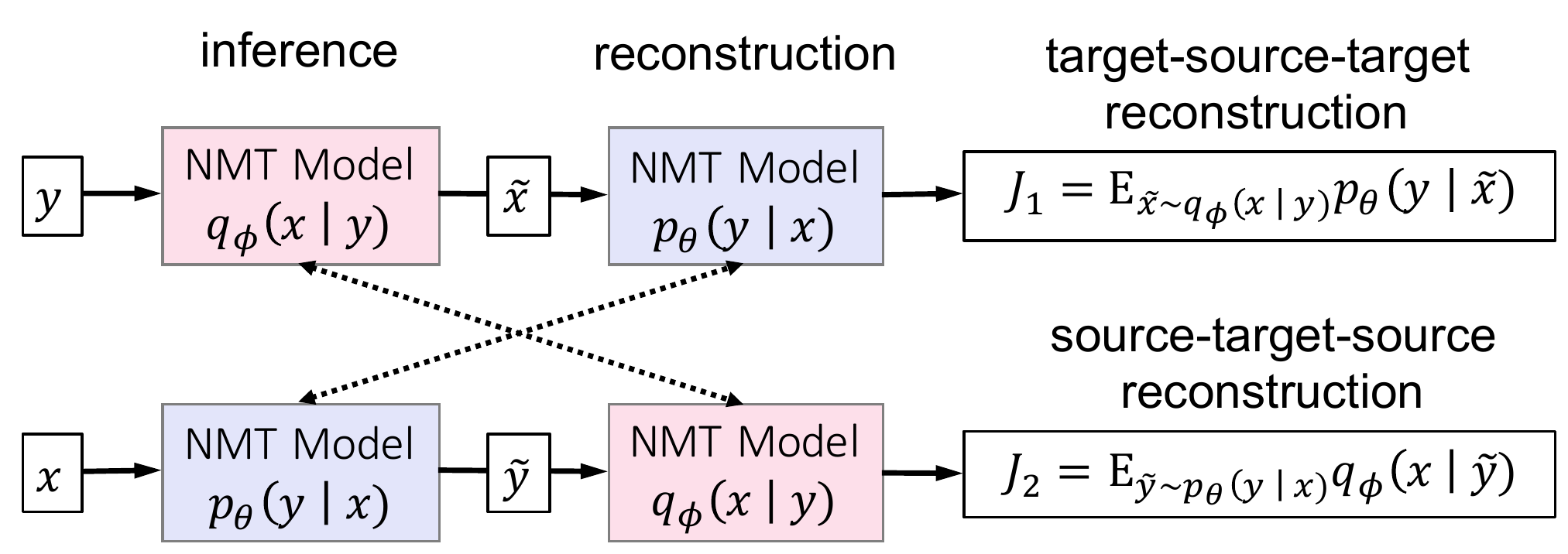}
\caption{Our dual reconstruction objective sums 1) a target-source-target objective~$\mathcal{J}_1$ on target sentences~$\boldsymbol{y}$ using the NMT model~$q_\phi(\boldsymbol{x}\, | \,\boldsymbol{y})$ for inference and~$p_\theta(\boldsymbol{y}\, | \,\boldsymbol{x})$ for reconstruction, and 2) a source-target-source objective~$\mathcal{J}_2$ on source  sentences~$\boldsymbol{x}$ using $p_\theta(\boldsymbol{y}\, | \,\boldsymbol{x})$ for inference and~$q_\phi(\boldsymbol{x}\, | \,\boldsymbol{y})$ for reconstruction. Models connected by dotted arrows share parameters.}
\label{fig:dual_reconstruction}
\end{figure}

Instead of proposing new methods, this paper sheds new light on how these established techniques work and how to use them. We introduce a {\bf dual reconstruction objective} to theoretically ground the comparison of semi-supervised training strategies that leverage monolingual data from both source and target languages (Figure~\ref{fig:dual_reconstruction}).
In Section~\ref{sec:theory}, we show that, under some assumptions, this objective remarkably shares the same global optimum as the intractable marginal likelihood objective where the model's marginal distribution~$p_\theta(\boldsymbol{y})$ coincides with the target sentence distribution~$p(\boldsymbol{y})$. We also show that Iterative Back-Translation (IBT) and Dual Learning can be viewed as different ways to approximate its optimization.

\looseness=-1
Theory suggests that IBT approximates the dual reconstruction objective more closely than the more complex Dual Learning approach, and in particular that Dual Learning's additional language model loss is redundant.
We investigate whether these differences matter in practice by conducting the first controlled empirical comparison of Back-Translation, IBT, and Dual Learning in high-resource (WMT de-en), low-resource (WMT tr-en), and cross-domain settings (News$\rightarrow$TED, de-en). Results support our theory that the additional language model loss and policy gradient estimation in Dual Learning is redundant and show that IBT outperforms the more complex Dual Learning algorithm in terms of translation quality.
Furthermore, we also compare different optimization strategies used in IBT to better balance translation quality against the computational cost.

\section{Background}
\paragraph{Notation} NMT models the probability of translating a source sequence~$\boldsymbol{x}$ into a target~$\boldsymbol{y}$ as $   p_{\theta}(\boldsymbol{y}\, | \,\boldsymbol{x}) = \prod_{t=1}^T p(y_t\, | \,\boldsymbol{y}_{<t}, \boldsymbol{x}; \theta)$ 
where~$\theta$ represents the model parameters, and~$T$ is the length of~$\boldsymbol{y}$~\cite{BahdanauCB15}.
The model computes the conditional probability of the next token at time~$t$ by~
$p(\cdot\, | \,\boldsymbol{y}_{<t}, \boldsymbol{x}; \theta) = \softmax(a(\boldsymbol{h}_t))$,
where~$a(\cdot)$ is a linear transformation, and~$\boldsymbol{h}_t$ is the hidden representation at step~$t$ usually modeled by an encoder-decoder network $\boldsymbol{h}_t = f(\boldsymbol{y}_{<t}, \boldsymbol{x})$. In supervised settings, NMT models are trained to maximize the likelihood of parallel sentence pairs:~$\mathcal{J}_{s}(\theta) = \sum_{(\boldsymbol{x}, \boldsymbol{y}) \in \mathcal{D}} \log p_{\theta}(\boldsymbol{y}\, | \,\boldsymbol{x})$ given data ~$\mathcal{D} = \{ (\boldsymbol{x}^{(n)}, \boldsymbol{y}^{(n)}) \}_{n=1}^N$.

\paragraph{IBT and Dual Learning} exploit large monolingual corpora which represent source and target language distributions better than the limited parallel corpora.
Back-Translation trains the source-to-target translation model~$p_\theta(\boldsymbol{y}\, | \,\boldsymbol{x})$ by maximizing the conditional log-likelihood of target language sentences~$\boldsymbol{y}$ given pseudo source sentences~$\boldsymbol{\tilde{x}}$ inferred by a pre-trained target-to-source translation model~$q_\phi(\boldsymbol{x}\, | \,\boldsymbol{y})$ given~$\boldsymbol{y}$.
IBT optimizes the dual translation models~$p_{\theta}(\boldsymbol{y}\, | \,\boldsymbol{x})$ and~$q_{\phi}(\boldsymbol{x}\, | \,\boldsymbol{y})$ via back-translation in turn, both for semi-supervised~\citep{ZhangLZC18,HoangKHC18,CotterellK18,NiuDC18} and unsupervised MT~\citep{ArtetxeLAC18,LampleCDR18,LampleOCDR18}.
Dual Learning takes the view of cooperative game theory where dual models collaborate with each other to learn to reconstruct the observed source and target monolingual sentences, and is widely used for semi-supervised~\citep{HeXQWYLM16}, unsupervised~\citep{WangXHTQZL2019}, and zero-shot multilingual NMT~\citep{SestorainCBH2018}. Concretely, Dual Learning optimizes~$p_\theta(\boldsymbol{y}\, | \,\boldsymbol{x})$ and~$q_\phi(\boldsymbol{x}\, | \,\boldsymbol{y})$ jointly
by reconstructing the original target sentence~$\boldsymbol{y}$ using~$p_\theta(\boldsymbol{y}\, | \,\boldsymbol{x})$ given the source~$\boldsymbol{\tilde{x}}$ inferred by~$q_\phi(\boldsymbol{x}\, | \,\boldsymbol{y})$, and vice versa. 
The reconstruction loss is augmented with a language model loss and used to update both reconstruction and inference models via policy gradient~\citep{Williams1992}.

While Dual Learning and IBT each improve BLEU over Back-Translation~\citep{ZhangLZC18,CotterellK18,HeXQWYLM16}, they have not been compared directly to each other. \citet{CotterellK18} interpret Back-Translation as a variational approximation where the pseudo source~$\boldsymbol{\tilde{x}}$ can be viewed as a latent variable and the target-to-source model~$q_{\phi}(\boldsymbol{x}\, | \,\boldsymbol{y})$ is an inference network that approximates the posterior distribution~$p_{\theta}(\boldsymbol{x}\, | \,\boldsymbol{y})$. Furthermore, they explain IBT as a way to better approximate the true posterior distribution with the target-to-source model. However, it is unclear how their heuristic objective relates to the ideal objective of maximizing the model's marginal likelihood of the target language monolingual data. More recently, \citet{HeWNB2020} connect back-translation and the language model loss in Dual Learning to the variational lower-bound (ELBO) of the marginal likelihood objective. We introduce a more direct connection.
\section{Theoretical View with Dual Reconstruction Objective}
\label{sec:theory}

\subsection{Variational Auto-Encoders for Semi-Supervised MT}
Following~\citet{CotterellK18}, we define a generative latent variable model of bitext
\[
    p_\theta(\boldsymbol{x}, \boldsymbol{y}) = p_\theta(\boldsymbol{y}\, | \,\boldsymbol{x}) q(\boldsymbol{x})
\]
where the source~$\boldsymbol{x}$ is randomly sampled from the prior distribution~$q(\boldsymbol{x})$ estimated by the empirical data distribution~$q_{data}(\boldsymbol{x})$ based on the abundant source monolingual data~$\mathcal{M}_X = \{\boldsymbol{x}^{(m)}\}_{m=1}^{M}$:
\[
    q_{data}(\boldsymbol{x}) =
    \begin{cases}
        \frac{1}{|\mathcal{M}_X|},& \text{if } \boldsymbol{x} \in \mathcal{M}_X \\
        0,              & \text{otherwise}
    \end{cases}
\]
and the target translation~$\boldsymbol{y}$ is sampled from the translation model~$p_\theta(\boldsymbol{y}\, | \,\boldsymbol{x})$ conditioned on~$\boldsymbol{x}$.

Given the target sentence distribution~$p(\boldsymbol{y})$ estimated by the empirical data distribution~$p_{data}(\boldsymbol{y})$ of target monolingual data~$\mathcal{M}_Y = \{\boldsymbol{y}^{(m)}\}_{m=1}^{M'}$, we can view~$\boldsymbol{x}$ as a latent variable and maximize the marginal log-likelihood
\[
    \mathcal{J}_{u}(\theta) = \e{\boldsymbol{y} \sim p(\boldsymbol{y})}{\log p_\theta(\boldsymbol{y})}
\]
where~$p_\theta(\boldsymbol{y})$ is the model's marginal likelihood~$p_\theta(\boldsymbol{y}) = \sum_{\boldsymbol{x}} p_\theta(\boldsymbol{x}, \boldsymbol{y})$.
The global optimum of the objective is achieved when the model's marginal distribution~$p_\theta(\boldsymbol{y})$ perfectly matches the target sentence distribution~$p(\boldsymbol{y})$.\footnote{We will define constraints to guarantee avoiding the uninteresting solution where $p_\theta(\boldsymbol{y}\, | \,\boldsymbol{x}) = p(\boldsymbol{y})$ in Section~\ref{sec:mi_constraint}.}

However, directly optimizing the marginal likelihood~$p_\theta(\boldsymbol{y})$ is intractable due to the infinite space of~$\boldsymbol{x}$. We can instead apply variational auto-encoding~(VAE) models by introducing an inference network~$p_{\psi}(\boldsymbol{x}\, | \,\boldsymbol{y})$ and maximize the variational lower-bound~(ELBO) of~$\log p_\theta(\boldsymbol{y})$:
\begin{equation}
\label{eq:elbo}
\begin{split}
    \log p_\theta(\boldsymbol{y}) \ge
    & \e{\boldsymbol{x} \sim p_{\psi}(\boldsymbol{x}\, | \,\boldsymbol{y})}{\log p_\theta(\boldsymbol{y}\, | \,\boldsymbol{x})} \\
    & - \kld{p_{\psi}(\boldsymbol{x}\, | \,\boldsymbol{y})}{q(\boldsymbol{x})} \\
\end{split}
\end{equation}
where~$\kld{p_\psi}{q}$ is the Kullback-Leibler (KL) divergence. However, estimating the prior distribution~$q(\boldsymbol{x})$ by the discrete data distribution~$q_{data}(\boldsymbol{x})$ makes it difficult to directly compute the KL term. One can estimate~$q(\boldsymbol{x})$ using a language model~(LM) trained to maximize the likelihood of the source monolingual data~\citep{Miao2016,Baziotis2019}, at the cost of introducing additional model bias into the translation model. The non-differentiable KL term requires gradient estimators such as policy gradient~\citep{Williams1992} or Gumbel-softmax~\citep{JangGP17}, which may introduce further training noise~\citep{HeWNB2020}.

To address these issues, we introduce the dual reconstruction objective, which includes two reconstruction terms that resemble the first term in the ELBO objective~(\cref{eq:elbo}) while excluding the KL term that is challenging to optimize and show that this objective has desirable properties and can be better approximated in practice.

\theoremstyle{definition}
\begin{definition}{}
Given prior distributions~$q(\boldsymbol{x})$ and~$p(\boldsymbol{y})$ over the sentences~$\boldsymbol{x}$ in the source language space~$\Sigma_x$ and~$\boldsymbol{y}$ in the target language space~$\Sigma_y$, we define the {\bf dual reconstruction objective}~$\mathcal{J}_{dual}(\theta, \phi)$ for dual translation models~$p_{\theta}(\boldsymbol{y}\, | \,\boldsymbol{x})$ and~$q_{\phi}(\boldsymbol{x}\, | \,\boldsymbol{y})$ as the sum of the target-source-target objective~$\mathcal{J}_1$ and source-target-source objective~$\mathcal{J}_2$:
\begin{equation}
\begin{split}
    & \mathcal{J}_{dual}(\theta, \phi) = \mathcal{J}_{1}(\theta, \phi) + \mathcal{J}_{2}(\theta, \phi) \\
    & \mathcal{J}_1(\theta, \phi) =
        \e{\boldsymbol{y} \sim p(\boldsymbol{y})} { \e{\boldsymbol{x} \sim q_\phi(\boldsymbol{x}\, | \,\boldsymbol{y})}{\log p_\theta(\boldsymbol{y}\, | \,\boldsymbol{x})} } \\
    & \mathcal{J}_2(\theta, \phi) =
        \e{\boldsymbol{x} \sim q(\boldsymbol{x})} { \e{\boldsymbol{y} \sim p_\theta(\boldsymbol{y}\, | \,\boldsymbol{x})}{\log q_\phi(\boldsymbol{x}\, | \,\boldsymbol{y})} }
\end{split}
\label{eq:dual_reconstruction}
\end{equation}
\end{definition}

For~$\mathcal{J}_1$, the target-to-source model~$q_\phi(\boldsymbol{x}\, | \,\boldsymbol{y})$ serves as the \textbf{inference model} to produce pseudo source sequences~$\boldsymbol{\tilde{x}}$ given target sequences~$\boldsymbol{y}$ and~$p_\theta(\boldsymbol{y}\, | \,\boldsymbol{x})$ serves as the \textbf{reconstruction model} to reconstruct~$\boldsymbol{y}$ given~$\boldsymbol{\tilde{x}}$, and vice versa for ~$\mathcal{J}_2$.
We first define the mutual information constraint in Section~\ref{sec:mi_constraint} and show in Section~\ref{sec:optimum} that $\mathcal{J}_{dual}(\theta, \phi)$ shares the same global optimum as the marginal likelihood objective which is intractable to optimize directly.\footnote{We focus on key components of the proof and leave detailed derivations for supplemental material.} In Section~\ref{sec:approx}, we compare and contrast how IBT and Dual Learning approximate $\mathcal{J}_{dual}(\theta, \phi)$.

\subsection{Mutual Information Constraint}
\label{sec:mi_constraint}

The global optimum of the marginal likelihood objective is achieved when the model's marginal distribution~$p_\theta(\boldsymbol{y}) = p(\boldsymbol{y})$. Given a translation model with enough capacity without any constraint on how the model output is dependent on the source context, this could lead to a degenerate solution~$p_\theta(\boldsymbol{y}\, | \,\boldsymbol{x}) = p(\boldsymbol{y})$ where the model ignores the source input and memorizes the monolingual training data. 
We constrain the translation model to avoid this situation, using the mutual information of a conditional distribution~$p_\theta(\boldsymbol{y}\, | \,\boldsymbol{x})$ which measures how much~$\boldsymbol{y}$ is dependent on~$\boldsymbol{x}$ in~$p_\theta$~\citep{HoffmanJ16}. Here, this mutual information measures the degree to which model translations depend on the source. 
\theoremstyle{definition}
\begin{definition}{}
Given a prior distribution~$q(\boldsymbol{x})$ over~$\boldsymbol{x} \in \Sigma_x$, we define the {\bf mutual information}~$\mi{p_\theta}$ of~$\boldsymbol{x}$ and~$\boldsymbol{y}$ in the conditional distribution~$p_\theta(\boldsymbol{y}\, | \,\boldsymbol{x})$:
\begin{equation}
\begin{split}
    \mi{p_\theta} = \e{\boldsymbol{x} \sim q(\boldsymbol{x})}{ \kld{ p_\theta(\boldsymbol{y}\, | \,\boldsymbol{x}) }{ p_\theta(\boldsymbol{y}) } }
\end{split}
\label{eq:mutual_info}
\end{equation}
where~$p_\theta(\boldsymbol{y})$ is the marginal distribution:
\begin{equation}
    p_\theta(\boldsymbol{y}) = \sum_{\boldsymbol{x}} p_{\theta}(\boldsymbol{y}\, | \,\boldsymbol{x}) q(\boldsymbol{x})
\end{equation}
\label{def:mutual_info}
\end{definition}

To avoid the degenerate solution, we constrain the model's mutual information by:
\[
    0 \le \mi{min} \le \mi{p_\theta} \le \mi{max} \le \max_{p \in P_{XY}} \mi{p}(\boldsymbol{x}; \boldsymbol{y})
\]
where~$\mi{min}$ and~$\mi{max}$ are pre-defined constant values between zero and the maximum mutual information between~$\boldsymbol{x}$ and~$\boldsymbol{y}$ given any joint distribution~$p(\boldsymbol{x}, \boldsymbol{y}) \in P_{XY}$ whose marginals satisfy~$\sum_{\boldsymbol{x}} p(\boldsymbol{x}, \boldsymbol{y}) = p(\boldsymbol{y})$ and~$\sum_{\boldsymbol{y}} p(\boldsymbol{x}, \boldsymbol{y}) = q(\boldsymbol{x})$.
\citet{Hledik2019MI} prove that the maximum mutual information~$\max_{p \in P_{XY}} \mi{p}(\boldsymbol{x}; \boldsymbol{y}) = \min(\h{}{q(\boldsymbol{x})}, \h{}{ p(\boldsymbol{y})})$, where~$\h{}{q(\boldsymbol{x})}$ and~$\h{}{ p(\boldsymbol{y})}$ are the entropy of prior distributions~$q(\boldsymbol{x})$ and~$p(\boldsymbol{y})$. Thus, the maximum mutual information should be large enough to properly bound the model's mutual information if~$q(\boldsymbol{x})$ and~$p(\boldsymbol{y})$ are defined on large monolingual corpora~$\mathcal{M}_X$ and~$\mathcal{M}_Y$.

Intuitively, the constraint requires that the model's mutual information cannot be so small that the model ignores the source context nor so large such that is not robust to the noise in the source input.
We will show in Section~\ref{sec:mi_analysis} that in practice, this constraint is met when jointly optimizing the supervised and unsupervised objectives without explicitly applying constrained optimization.

\subsection{Understanding the Global Optimum of the Dual Reconstruction Objective}
\label{sec:optimum}
We first characterize the upper bound of the dual reconstruction objective.

\begin{prop}
\label{prop}
Given prior distributions~$q(\boldsymbol{x})$ and~$p(\boldsymbol{y})$ over~$\boldsymbol{x} \in \Sigma_x$ and~$\boldsymbol{y} \in \Sigma_y$, if parameterized probability models~$p_\theta$ and~$q_\phi$ have enough capacity under the constraint that:
\begin{equation*}
\begin{split}
    & 0 \le \mi{min} \le \mi{p_\theta}, \mi{q_\phi} \le \mi{max} \le \max_{p \in P_{XY}} \mi{p}(\boldsymbol{x}; \boldsymbol{y}) \\
\end{split}
\end{equation*}
where~$\mi{min}$ and~$\mi{max}$ are pre-defined constant values between zero and the maximum mutual information between~$\boldsymbol{x}$ and~$\boldsymbol{y}$ given any joint distribution~$p(\boldsymbol{x}, \boldsymbol{y}) \in P_{XY}$ whose marginals satisfy~$\sum_{\boldsymbol{x}} p(\boldsymbol{x}, \boldsymbol{y}) = p(\boldsymbol{y})$ and~$\sum_{\boldsymbol{y}} p(\boldsymbol{x}, \boldsymbol{y}) = q(\boldsymbol{x})$. Then, the dual reconstruction objective is upper-bounded by~$\mathcal{J}_{dual}(\theta, \phi) \le 2\mi{max} - \h{}{q(\boldsymbol{x})} - \h{}{p(\boldsymbol{y})}$, and the upper bound is achieved iff
\begin{equation}
\label{eq:optcrit}
\begin{split}
    & \mi{q_{\phi}} = \mi{max} \\
    & \mi{p_{\theta}} = \mi{max} \\
    & p_{\theta}(\boldsymbol{y}\, | \,\boldsymbol{x})
        = \frac{q_{\phi}(\boldsymbol{x}\, | \,\boldsymbol{y})}{q_{\phi}(\boldsymbol{x})} p(\boldsymbol{y}) \\
    & q_{\phi}(\boldsymbol{x}\, | \,\boldsymbol{y})
        = \frac{p_{\theta}(\boldsymbol{y}\, | \,\boldsymbol{x})}{p_{\theta}(\boldsymbol{y})} q(\boldsymbol{x}) \\
\end{split}
\end{equation}
\end{prop}

\begin{proof}
First we prove that~$\mathcal{J}_1(\theta, \phi) \le \mi{max} - \h{}{p(\boldsymbol{y})}$, and the upper bound is achieved iff
\begin{equation*}
\begin{split}
    & \mi{q_{\phi}} = \mi{max} \\
    & p_{\theta}(\boldsymbol{y}\, | \,\boldsymbol{x})
        = \frac{q_{\phi}(\boldsymbol{x}\, | \,\boldsymbol{y})}{q_{\phi}(\boldsymbol{x})} p(\boldsymbol{y})
\end{split}
\end{equation*}

To show this, we denote the posterior distribution~$Q(\boldsymbol{y}\, | \,\boldsymbol{x}) = \frac{q_{\phi}(\boldsymbol{x}\, | \,\boldsymbol{y})}{q_{\phi}(\boldsymbol{x})} p(\boldsymbol{y})$, and rewrite~$\mathcal{J}_1$:
\begin{equation*}
\begin{split}
    \mathcal{J}_1
    =& \e{\boldsymbol{y} \sim p(\boldsymbol{y})} { \e{\boldsymbol{x} \sim q_\phi(\boldsymbol{x}\, | \,\boldsymbol{y})}{\log p_\theta(\boldsymbol{y}\, | \,\boldsymbol{x})} } \\
    =& \mi{q_\phi} - \h{}{p(\boldsymbol{y})} \\
    &- \kld{q_\phi(\boldsymbol{x}\, | \,\boldsymbol{y}) p(\boldsymbol{y})}{p_\theta(\boldsymbol{y}\, | \,\boldsymbol{x}) q_\phi(\boldsymbol{x})} \\
\end{split}
\end{equation*}
Since the KL divergence between two distributions is always non-negative and is zero iff they are equal, we have
\begin{equation*}
\begin{split}
    \mathcal{J}_1(\theta, \phi)
    \le \mi{q_\phi} - \h{}{p(\boldsymbol{y})}
    \le \mi{max} - \h{}{p(\boldsymbol{y})}
\end{split}
\end{equation*}
and~$\mathcal{J}_1(\theta, \phi) = \mi{max} - \h{}{p(\boldsymbol{y})}$ iff
\begin{equation*}
\begin{split}
    & \mi{q_{\phi}} = \mi{max} \\
    & \kld{q_{\phi}(\boldsymbol{x}\, | \,\boldsymbol{y}) p(\boldsymbol{y})}{p_{\theta}(\boldsymbol{y}\, | \,\boldsymbol{x}) q_{\phi}(\boldsymbol{x})} = 0 \\
\end{split}
\end{equation*}
The second equality holds iff
\begin{equation*}
    p_{\theta}(\boldsymbol{y}\, | \,\boldsymbol{x})
    = \frac{q_{\phi}(\boldsymbol{x}\, | \,\boldsymbol{y})}{q_{\phi}(\boldsymbol{x})} p(\boldsymbol{y})
\end{equation*}

Similarly, we can prove that~$\mathcal{J}_2(\theta, \phi) \le \mi{max} - \h{}{q(\boldsymbol{x})}$, and the upper bound is achieved iff
\begin{equation*}
\begin{split}
    & \mi{p_{\theta}} = \mi{max} \\
    & q_{\phi}(\boldsymbol{x}\, | \,\boldsymbol{y})
        = \frac{p_{\theta}(\boldsymbol{y}\, | \,\boldsymbol{x})}{p_{\theta}(\boldsymbol{y})} q(\boldsymbol{x}) \\
\end{split}
\end{equation*}
thus~$\mathcal{J}_{dual}(\theta, \phi) \le 2\mi{max} - \h{}{q(\boldsymbol{x})} - \h{}{p(\boldsymbol{y})}$ and the upper bound is achieved iff~$\theta$ and~$\phi$ satisfy~\cref{eq:optcrit}, concluding the proof.
\end{proof}

\cref{prop} shows that~$\mathcal{J}_{dual}(\theta, \phi)$ has an upper bound that could be reached when the mutual information of~$p_\theta(\boldsymbol{y}\, | \,\boldsymbol{x})$ and~$q_\phi(\boldsymbol{x}\, | \,\boldsymbol{y})$ are maximized, and~$p_\theta(\boldsymbol{y}\, | \,\boldsymbol{x})$ and~$q_\phi(\boldsymbol{x}\, | \,\boldsymbol{y})$ are equal to the posterior distribution for each other.
Next we show that the upper bound is indeed the global maximum of the objective~$\mathcal{J}_{dual}(\theta, \phi)$, as there exists a solution for the above conditions~(proof in Appendix~\ref{sec:appendix:proof_2}).

\begin{prop}
\label{prop:construct}
Given distributions~$q(\boldsymbol{x})$ and~$p(\boldsymbol{y})$ over~$\boldsymbol{x} \in \Sigma_x$ and~$\boldsymbol{y} \in \Sigma_y$, if parameterized probability models~$p_\theta$ and~$q_\phi$ have enough capacity under the constraint that:
\begin{equation}
\label{eq:constraint}
\begin{split}
    & 0 \le \mi{min} \le \mi{p_\theta}, \mi{q_\phi} \le \mi{max} \le \max_{p \in P_{XY}} \mi{p}(\boldsymbol{x}; \boldsymbol{y}) \\
\end{split}
\end{equation}
where~$\mi{min}$ and~$\mi{max}$ are pre-defined constant values between zero and the maximum mutual information between~$\boldsymbol{x}$ and~$\boldsymbol{y}$ given any joint distribution~$p(\boldsymbol{x}, \boldsymbol{y}) \in P_{XY}$ whose marginals satisfy~$\sum_{\boldsymbol{x}} p(\boldsymbol{x}, \boldsymbol{y}) = p(\boldsymbol{y})$ and~$\sum_{\boldsymbol{y}} p(\boldsymbol{x}, \boldsymbol{y}) = q(\boldsymbol{x})$.
Then there exist~$\theta^*$ and~$\phi^*$ such that:
\begin{equation}
\label{eq:crit}
\begin{split}
    & \mi{q_{\phi^*}} = \mi{p_{\theta^*}} = \mi{max} \\
    & p_{\theta^*}(\boldsymbol{y}\, | \,\boldsymbol{x})
        = \frac{q_{\phi^*}(\boldsymbol{x}\, | \,\boldsymbol{y})}{q_{\phi^*}(\boldsymbol{x})} p(\boldsymbol{y}) \\
    & q_{\phi^*}(\boldsymbol{x}\, | \,\boldsymbol{y})
        = \frac{p_{\theta^*}(\boldsymbol{y}\, | \,\boldsymbol{x})}{p_{\theta^*}(\boldsymbol{y})} q(\boldsymbol{x}) \\
\end{split}
\end{equation}
\end{prop}

Finally, we connect the global optimum of the dual reconstruction objective to that of the marginal likelihood objective~(proof in Appendix~\ref{sec:appendix:theorem}).
\begin{theorem}
\label{theorem}
Given prior distributions~$q(\boldsymbol{x})$ and~$p(\boldsymbol{y})$ over~$\boldsymbol{x} \in \Sigma_x$ and~$\boldsymbol{y} \in \Sigma_y$, if parameterized probability models~$p_\theta$ and~$q_\phi$ have enough capacity under the constraint that:
\begin{equation*}
\begin{split}
    & 0 \le \mi{min} \le \mi{p_\theta}, \mi{q_\phi} \le \mi{max} \le \max_{p \in P_{XY}} \mi{p}(\boldsymbol{x}; \boldsymbol{y}) \\
\end{split}
\end{equation*}
where~$\mi{min}$ and~$\mi{max}$ are pre-defined constant values between zero and the maximum mutual information between~$\boldsymbol{x}$ and~$\boldsymbol{y}$ given any joint distribution~$p(\boldsymbol{x}, \boldsymbol{y}) \in P_{XY}$ whose marginals satisfy~$\sum_{\boldsymbol{x}} p(\boldsymbol{x}, \boldsymbol{y}) = p(\boldsymbol{y})$ and~$\sum_{\boldsymbol{y}} p(\boldsymbol{x}, \boldsymbol{y}) = q(\boldsymbol{x})$. Let $\theta^*, \phi^*$ be the global optimum of the dual reconstruction objective~$\max_{\theta, \phi} \mathcal{J}_{dual}(\theta, \phi)$, then~$q_{\phi^*}(\boldsymbol{x}) = q(\boldsymbol{x})$,~$p_{\theta^*}(\boldsymbol{y}) = p(\boldsymbol{y})$, and~$\mi{q_{\phi^*}} = \mi{p_{\theta^*}} = \mi{max}$.
\end{theorem}

Thus, while the marginal likelihood objective provides no guarantee for the model's mutual information, the global optimum of dual reconstruction objective guarantees that the mutual information of translation models~$p_\theta(\boldsymbol{y}\, | \,\boldsymbol{x})$ and~$q_\phi(\boldsymbol{x}\, | \,\boldsymbol{y})$ will be maximized to~$\mi{max}$.

\subsection{Practical Approximations}
\label{sec:approx}

Despite its desirable optimum, the dual reconstruction objective cannot be directly optimized since decoding is not differentiable. We compare how it is approximated by IBT vs. Dual Learning.
 
\paragraph{Gradient Approximation} To estimate the dual reconstruction objective, one could use sampling or beam search from the model distribution. However, since neither approach is differentiable, the gradients~$\nabla_\theta \mathcal{J}_2$ and~$\nabla_\phi \mathcal{J}_1$ cannot be computed directly. IBT blocks the gradients~$\nabla_\theta \mathcal{J}_2$ and~$\nabla_\phi \mathcal{J}_1$ assuming that they are negligible, while Dual Learning approximates them by policy gradient~\citep{Williams1992}, which can lead to slow and unstable training~\citep{HendersonIBPPM18,WuTQLL18}.
\cref{prop} shows that the objective is maximized when the mutual information is maximized to~$I_{max}$.
Thus, maximizing the mutual information by other means can help side-step this issue. For example, combining the supervised and unsupervised training objectives \citep{SennrichHB16,CotterellK18} to train models jointly on the parallel and monolingual data can help.
For unsupervised MT, the denoising auto-encoding objective introduced in \citet{LampleCDR18} can be viewed as a way to maximize the mutual information.

\paragraph{LM Loss} Dual Learning combines the dual reconstruction objective with an LM loss to encourage the generated translations to be close to the target language domain. \cref{theorem} suggests that the LM loss is redundant: optimizing the dual reconstruction objective implicitly pushes the output distributions of the source-to-target and target-to-source models toward the target and source language distributions respectively, which has the same effect intended by the LM loss.

\paragraph{Optimization Strategy} 
\looseness=-1
While Dual Learning uses batch-level updates, where back-translations are generated on-the-fly and the translation models~$p_\theta$ and~$q_\phi$ are updated alternately in data batches, IBT adopts different strategies based on the data settings. Batch-level IBT is used in unsupervised MT to quickly boost the model performance from a cold start~\citep{ArtetxeLAC18,LampleCDR18}, while epoch-level IBT is used in semi-supervised MT, where a fixed model~$p_\theta$ is used to back-translate the entire monolingual corpus to train~$q_\phi$ until convergence and vice-versa for~$p_\theta$~\citep{ZhangLZC18}.
 
\paragraph{Summary} This theoretical analysis suggests that the dual reconstruction objective is a good alternative to the intractable marginal likelihood objective, and that IBT approximates it more closely than the more complex Dual Learning objective. However, we do not know whether the Dual Reconstruction optimum is reached in practice. We therefore conduct an extensive empirical study to determine whether the differences in approximations made by IBT and Dual Learning matter.

\begin{table*}
\centering
\begin{tabular}{lcccrrcc}
\toprule
\textbf{Task} & Lang. & \multicolumn{2}{c}{Parallel Data} & \multicolumn{2}{c}{Mono. Data} & {Validation} & {Test} \\\hline
{\bf high-resource} & {de-en} & {News} & 4.5M & {News}  & {5.0M} & {newstest15} & {newstest16-18} \\
{\bf low-resource} & {tr-en} & {News} & {0.2M}& {News}  & {0.8M} & {newstest16} & {newstest17-18} \\
{\bf cross-domain} & {de-en} & {News} & {4.5M} & {TED} & {0.5M} & {iwslt-test14} & {iwslt-test15-17} \\\hline
\end{tabular}
\caption{The empirical comparison spans three data conditions (and both translation directions). We report provenance and the number of sentences in parallel and monolingual training data, as well as validation and test sets for each setting. Monolingual data are randomly selected from \textit{``News Crawl: articles from 2015''} for German$\leftrightarrow$English and \textit{``News Crawl: articles from 2017''} for Turkish$\leftrightarrow$English, and TED talks data for TED.}
\label{tab:exp_settings}
\end{table*}

\section{Empirical Study}
\label{sec:experiments}

\begin{table*}[hp]
\begin{subtable}{.9\linewidth}
\begin{tabular}{lcrp{0.07\textwidth}cccp{0.06\textwidth}ccc}
\toprule
\multirow{2}{*}{\bf Low-Resource} & \multirow{2}{*}{$\alpha_{LM}$} & \multirow{2}{*}{hours} & \multicolumn{4}{c}{tr-en \abr{bleu}} & \multicolumn{4}{c}{en-tr \abr{bleu}} \\
& & & & {2017} & {2018} & {Avg} & & {2017} & {2018} & {Avg} \\
\midrule
{baseline} & {--} & {8.0} & & {15.14} & {15.95} & {15.55} & & {11.17} & {10.18} & {10.68} \\
\midrule
{epoch-level \abr{ibt}-1} & {--} & {86.1} & & {16.36} & {16.44} & {16.40} & & {15.08} & {12.98} & {\bf 14.03} \\
{epoch-level \abr{ibt}-2} & {--} & {162.2} & & {19.12} & {19.63} & {\bf 19.38} & & {14.94} & {12.53} & {13.74} \\
{epoch-level \abr{ibt}-3} & {--} & {237.5} & & {18.76} & {19.01} & {18.89} & & {15.04} & {12.93} & {13.99} \\
{batch-level \abr{ibt}} & {--} & {160.6} & & {17.18} & {18.08} & {17.63} & & {13.90} & {11.84} & {12.87} \\
{Dual Learning} & {0.0} & {313.2} & & {17.07} & {18.00} & {17.54} & & {14.17} & {11.91} & {13.04} \\
{Dual Learning} & {0.1} & {257.8} & & {17.09} & {17.62} & {17.36} & & {13.88} & {11.49} & {12.69} \\
{Dual Learning} & {0.5} & {421.2} & & {17.33} & {18.36} & {17.85} & & {14.54} & {12.30} & {13.42} \\
\end{tabular}
\label{tab:low_bleu}
\end{subtable}

\begin{subtable}{.9\linewidth}
\begin{tabular}{lcrcccccccc}
\toprule
\multirow{2}{*}{\bf High-Resource} & \multirow{2}{*}{$\alpha_{LM}$} & \multirow{2}{*}{hours} & \multicolumn{4}{c}{de-en \abr{bleu}} & \multicolumn{4}{c}{en-de \abr{bleu}} \\
& & & {2016} & {2017} & {2018} & {Avg} & {2016} & {2017} & {2018} & {Avg} \\
\midrule
{baseline} & {--} & {26.7} & {31.95} & {27.74} & {34.59} & {31.43} & {29.18} & {23.46} & {34.53} & {29.06} \\
\midrule
{epoch-level \abr{ibt}-1} & {--} & {439.0} & {32.59} & {28.46} & {35.22} & {32.09} & {30.13} & {23.87} & {35.35} & {\bf 29.78} \\
{epoch-level \abr{ibt}-2} & {--} & {850.9} & {33.64} & {29.13} & {36.37} & {\bf 33.05} & {29.99} & {24.42} & {35.60} & {\bf 30.00} \\
{epoch-level \abr{ibt}-3} & {--} & {1261.6} & {33.43} & {29.07} & {36.17} & {32.89} & {29.93} & {24.24} & {35.46} & {\bf 29.88} \\
{batch-level \abr{ibt}} & {--} & {94.0} & {32.95} & {28.65} & {35.24} & {32.28} & {29.70} & {23.78} & {34.89} & {29.46} \\
{Dual Learning} & {0.0} & {128.2} & {32.79} & {28.47} & {35.10} & {32.12} & {29.37} & {23.50} & {34.67} & {29.18} \\
{Dual Learning} & {0.1} & {93.3} & {32.63} & {28.47} & {34.88} & {31.99} & {29.38} & {23.79} & {34.71} & {29.29} \\
{Dual Learning} & {0.5} & {152.1} & {32.89} & {28.69} & {35.32} & {32.30} & {29.58} & {23.65} & {34.88} & {29.37} \\
\end{tabular}
\label{tab:high_bleu}
\end{subtable}

\begin{subtable}{.9\linewidth}
\begin{tabular}{lcrcccccccc}
\toprule
\multirow{2}{*}{\bf Cross-Domain} & \multirow{2}{*}{$\alpha_{LM}$} & \multirow{2}{*}{hours} & \multicolumn{4}{c}{de-en \abr{bleu}} & \multicolumn{4}{c}{en-de \abr{bleu}} \\
& & & {2015} & {2016} & {2017} & {Avg} & {2015} & {2016} & {2017} & {Avg} \\
\midrule
{baseline} & {--} & {26.2} & {27.11} & {27.37} & {23.65} & {26.04} & {26.35} & {23.10} & {21.69} & {23.71} \\
\midrule
{epoch-level \abr{ibt}-1} & {--} & {71.1} & {28.88} & {28.73} & {25.37} & {\bf 27.66} & {26.69} & {24.02} & {22.59} & {24.43} \\
{epoch-level \abr{ibt}-2} & {--} & {115.0} & {28.70} & {28.72} & {25.37} & {\bf 27.60} & {27.57} & {24.50} & {22.78} & {\bf 24.95} \\
{epoch-level \abr{ibt}-3} & {--} & {159.8} & {29.13} & {29.00} & {25.33} & {\bf 27.82} & {27.31} & {24.37} & {22.92} & {\bf 24.87} \\
{batch-level \abr{ibt}} & {--} & {45.0} & {28.03} & {27.78} & {24.53} & {26.78} & {26.84} & {23.64} & {22.35} & {24.28} \\
{Dual Learning} & {0.0} & {65.8} & {28.04} & {27.73} & {24.36} & {26.71} & {26.70} & {23.85} & {22.21} & {24.25} \\
{Dual Learning} & {0.1} & {59.3} & {27.77} & {27.84} & {24.51} & {26.71} & {26.99} & {23.86} & {22.59} & {24.48} \\
{Dual Learning} & {0.5} & {92.7} & {27.84} & {28.00} & {24.18} & {26.67} & {27.23} & {24.08} & {22.72} & {24.68} \\
\bottomrule
\end{tabular}
\label{tab:xdomain_bleu}
\end{subtable}
\caption{BLEU scores and total training time~(\textit{hours}) on the low-resource, high-resource, and cross-domain tasks. \textit{epoch-level} \abr{ibt}-1, \abr{ibt}-2, and \abr{ibr}-3 denotes models fine-tuned with IBT for 1--3 iterations, and~$\alpha_{LM}$ denotes the weight for the LM loss. We boldface the highest average scores and their ties based on the significance test. Overall, epoch-level IBT outperforms all other methods at the cost of much longer training time.}
\label{tab:main_result}
\end{table*}

We evaluate on six translation tasks (Table~\ref{tab:exp_settings}), including  German$\leftrightarrow$English~(de-en), Turkish$\leftrightarrow$English~(tr-en) from WMT18~\citep{BojarFFGHHKM18}, and a cross-domain task which tests de-en models trained on WMT data on the TED test sets from IWSLT17~\citep{IWSLT17}.\footnote{We exclude Rapid and ParaCrawl corpora as they are noisy and thus require data filtering~\cite{Morishita2018}.} 

\subsection{Model and Training Configuration}
We adopt the base Transformer model~\citep{Vaswani2017}.
We pre-train models with the supervised objective until convergence, and fine-tune on the mixed parallel and monolingual data as in prior work~\citep{SennrichHB16,CotterellK18}. We use the Adam optimizer~\citep{KingmaB15} with a batch size of~$32$ sentences and checkpoint the model every~$2500$ updates. 
At decoding time, we use beam search with a beam size of~$5$. 
The LMs in Dual Learning are RNNs~\citep{MikolovKBCK10} with~$512$ hidden units. All model and training details are in Appendix~\ref{sec:appendix:exp}.

For preprocessing, we normalize punctuations and apply tokenization, true-casing, and joint source-target Byte Pair Encoding~\citep{SennrichHB16bpe} with~$32,000$ operations. We set the maximum sentence length to~$50$.

\subsection{Baselines and Evaluation}
Our experiments are based on strong supervised baselines.\footnote{de-en: 2--4 BLEU higher than the baseline of~\citet{Morishita2018}; tr-en: on par or higher than the baseline of~\citet{LIUM17}.} We compare semi-supervised models that are fine-tuned with Back-Translation, epoch-level and batch-level IBT, and Dual Learning with varying interpolation weights~$\alpha_{LM} = \{0, 0.1, 0.5\}$ for the LM loss.\footnote{By contrast, prior work only reports results for ~$\alpha_{LM} = 0.005$~\citep{HeXQWYLM16}. Our preliminary result show that~$\alpha_{LM} = 0.005$ obtains similar results to~$\alpha_{LM} = 0$.} 
Following \citet{HeXQWYLM16}, we use beam search with a beam size of~$2$ for inference in Dual Learning and IBT. 

We evaluate translation quality using sacreBLEU\footnote{Version: BLEU+case.mixed+numrefs.1+smooth.exp+\\tok.13a+version.1.2.11} and total training time in hours. We also show learning curves  for the approximated dual reconstruction loss~(negative of the dual reconstruction objective in \cref{eq:dual_reconstruction}, averaged over the training batches from both directions).

\subsection{Findings}

\paragraph{Overview} 
\looseness=-1
All semi-supervised training techniques improve translation quality over the supervised-only baseline (Table~\ref{tab:main_result}). The first iteration of IBT (i.e. Back-Translation) on monolingual data improves over the baseline significantly\footnote{All mentions of statistical significance are based on a paired Student's t-test with $p <0.05$.} by 0.7--3.4 BLEU. IBT is more effective in the direction where the model in the opposite direction is most improved by Back-Translation 
For example, in the high and low resource tasks where Back-Translation improves over the baseline more when translating out of English, the best performing IBT model significantly improves BLEU over Back-Translation when translating into English, but not in the other direction. In the cross-domain scenario where Back-Translation improves more on de-en, IBT outperforms Back-Translation on en-de, but the improvement is not significant in the other direction.

\paragraph{Impact of Policy Gradient}
Updating the inference model via policy gradient fails to lower the dual reconstruction loss and has little impact on BLEU. We compare Dual Learning (with $\alpha_{LM} = 0$) to batch-level IBT, so that the only difference between the two approaches is whether the inference model is updated. Batch-level IBT achieves similar or higher BLEU than Dual Learning for all tasks, except for the low-resource en-tr task where the BLEU difference is small~($<0.2$).
In addition, batch-level IBT trains ~30--50\% faster than Dual Learning. Figure~\ref{fig:rec_curve} shows that the policy gradient update has little impact on the dual reconstruction loss on all tasks.

\paragraph{Impact of LM}
The best Dual Learning BLEU is obtained with~$\alpha_{LM} = 0.5$ on all tasks except for de-en in the cross-domain setting (Table~\ref{tab:main_result}). However, it brings only small BLEU improvements~(0.2--0.4) over Dual Learning without LM loss~($\alpha_{LM} > 0$), but causes the dual reconstruction loss to decrease slower~(Figure~\ref{fig:rec_curve}), and slows down training by~20--40\%. In all cases, IBT outperforms Dual Learning.

\paragraph{Epoch vs. Batch IBT} 
The best epoch-level IBT model outperforms batch-level IBT by~0.5--1.8 BLEU overall, at the cost of much slower training:~13 times longer in the high-resource setting,~1.5 times longer in the low-resource setting, and~3.5 times longer in the cross-domain setting. Running IBT for two iterations is a good choice to balance training efficiency and translation quality, as the third iteration does not help BLEU.

\begin{figure}[!t]
    \centering
    \begin{subfigure}[b]{0.4\textwidth}
        \centering
        \includegraphics[width=\textwidth]{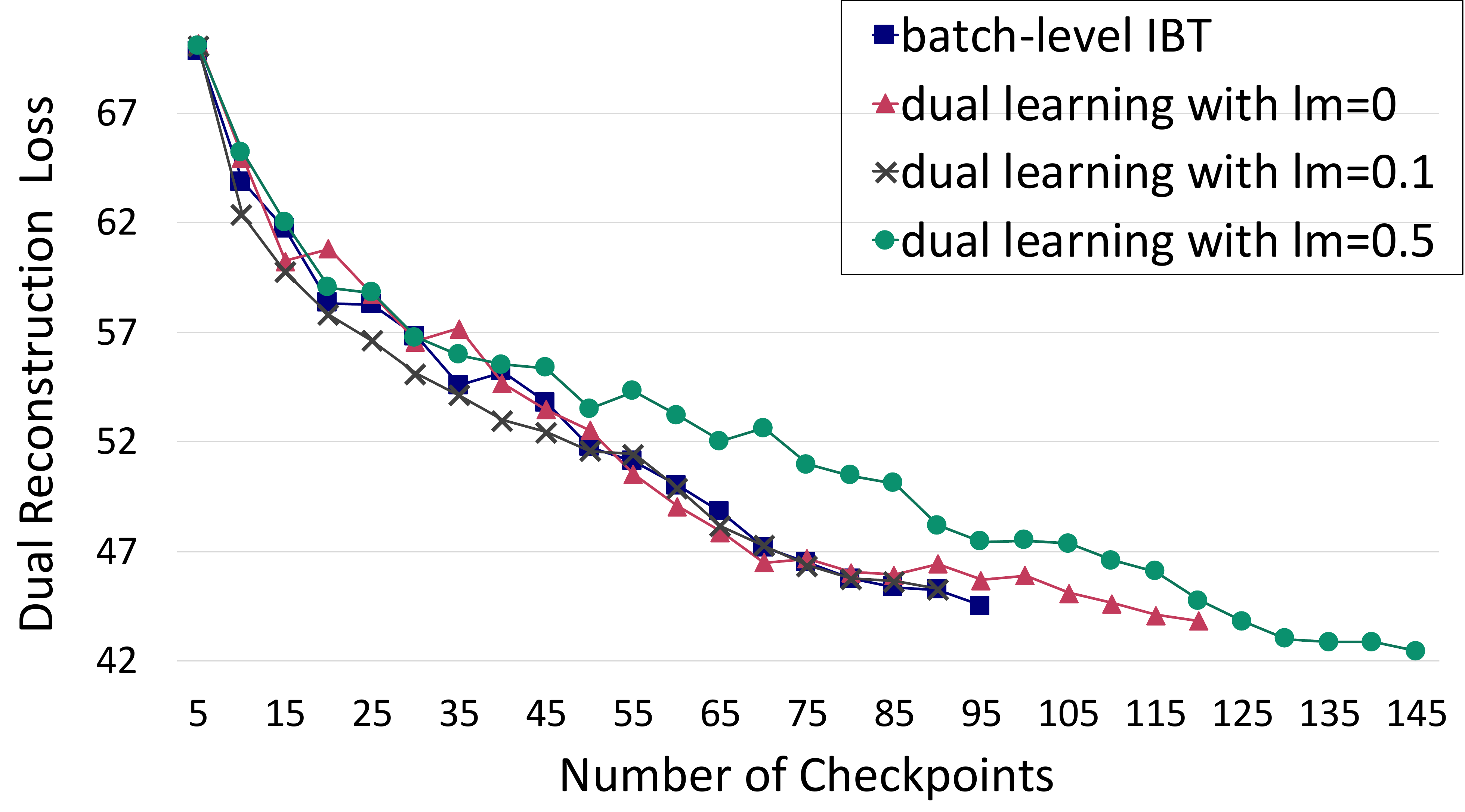}
        \caption{Low-resource task}
    \end{subfigure}
	\hfill
	
	\centering
    \begin{subfigure}[b]{0.4\textwidth}
        \centering
        \includegraphics[width=\textwidth]{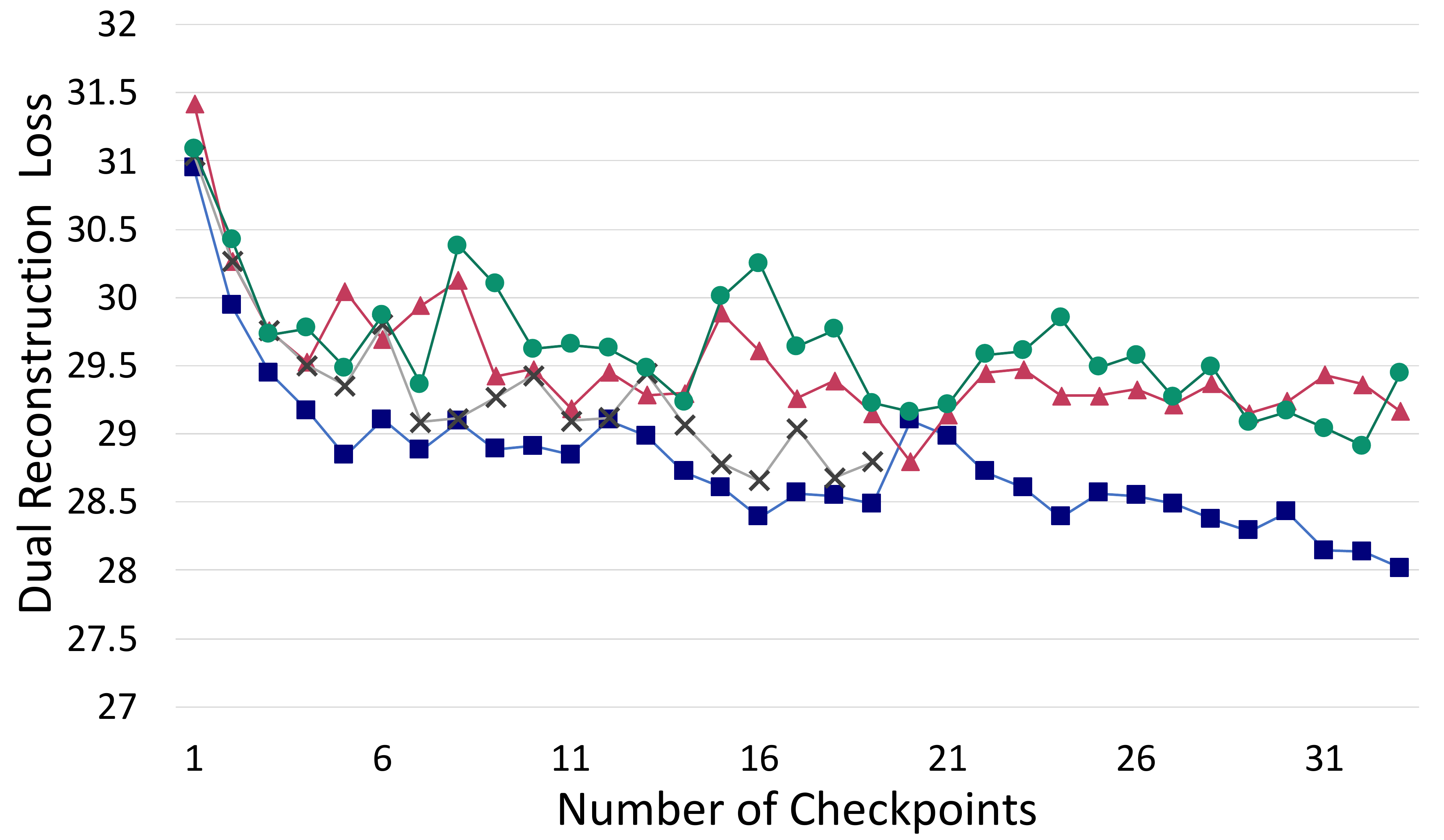}
        \caption{High-resource task}
    \end{subfigure}
	\hfill
	
	\centering
    \begin{subfigure}[b]{0.4\textwidth}
        \centering
        \includegraphics[width=\textwidth]{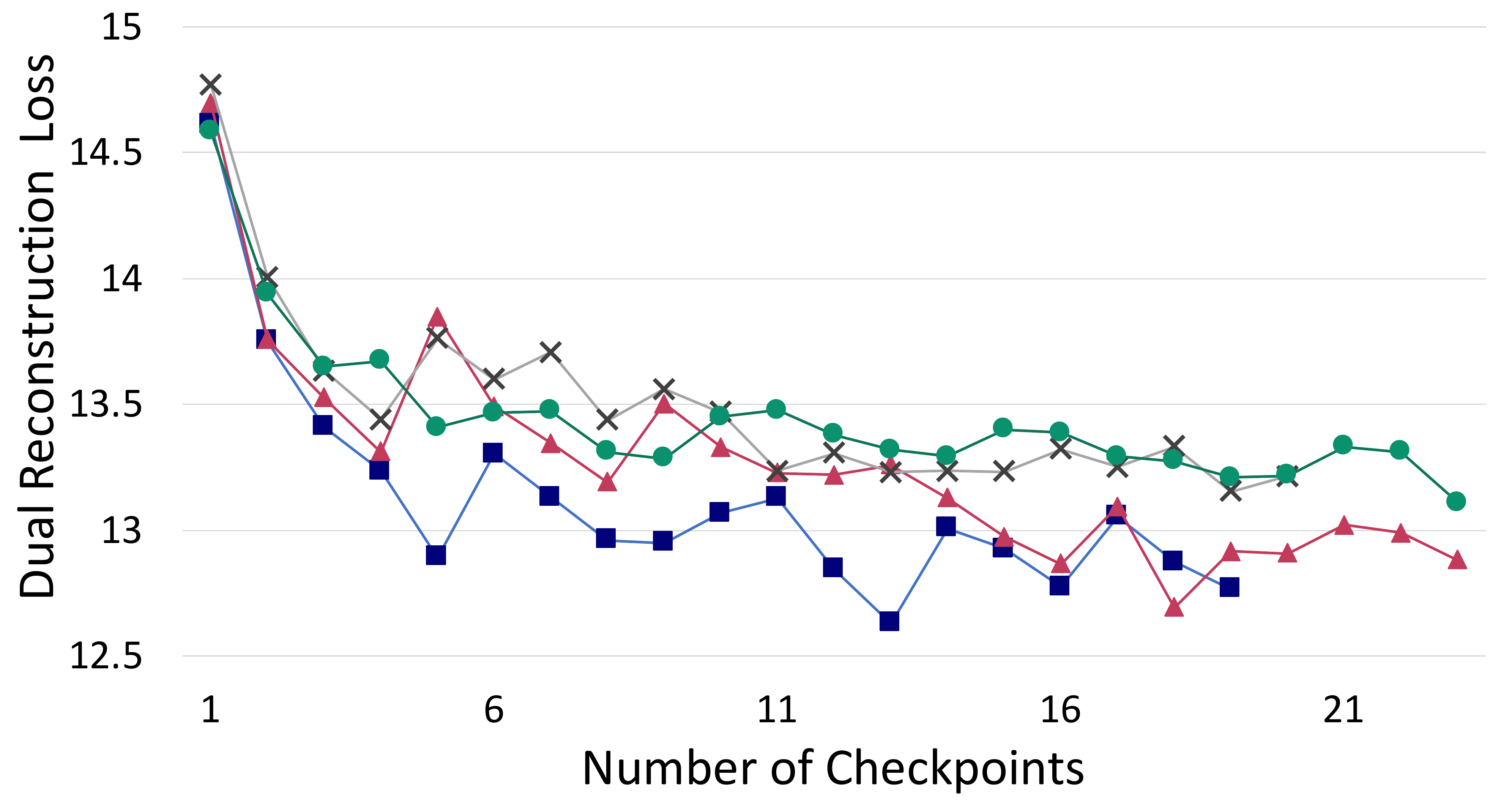}
        \caption{Cross-domain task}
    \end{subfigure}
\caption{Learning curves for the approximated dual reconstruction loss averaged over the training batches from both directions on the low-resource, high-resource, and cross-domain tasks.}
\label{fig:rec_curve}
\end{figure}

\subsection{Mutual Information Analysis}
\label{sec:mi_analysis}

We test the hypothesis that the mutual information constraint is met when training models on the combined supervised and unsupervised objectives in the low-resource setting (the most adversarial condition with the fewest supervised training samples).

The mutual information~$\mi{p_\theta}$ from Definition~\ref{def:mutual_info} can be computed by~\citet{HoffmanJ16}:
\begin{equation}
\begin{split}
    \mi{p_\theta} = & \e{\boldsymbol{x} \sim q(\boldsymbol{x})}{ \kld{ p_\theta(\boldsymbol{y}\, | \,\boldsymbol{x}) }{ p(\boldsymbol{y}) } } \\
    & - \kld{p_\theta(\boldsymbol{y})}{p(\boldsymbol{y})}
\end{split}
\end{equation}
where prior distributions~$q(\boldsymbol{x})$ and~$p(\boldsymbol{y})$ are estimated by the empirical data distribution given the monolingual corpora~$\mathcal{M}_X$ and~$\mathcal{M}_Y$. Although computing~$I_{p_\theta}$ directly is intractable, it can be approximated with a Monte Carlo estimate. Following~\citet{Dieng2019}, we approximate the two KL terms by Monte Carlo, where samples from~$p_\theta(\boldsymbol{y})$ can be obtained by ancestral sampling~(we use beam search with beam size of five to sample from~$p_\theta(\boldsymbol{y}\, | \,\boldsymbol{x})$). The marginal probability~$p_\theta(\boldsymbol{y}) = \e{\boldsymbol{x} \sim q(\boldsymbol{x})}{p_\theta(\boldsymbol{y}\, | \,\boldsymbol{x})}$ can also be estimated by Monte Carlo. Due to data sparsity, the conditional likelihood~$p_\theta(\boldsymbol{y}\, | \,\boldsymbol{x})$ will be near zero for most source sentences randomly sampled from~$q(\boldsymbol{x})$. To better estimate it, we smooth the data distribution of the original dataset~$\mathcal{D}$ by generating a randomly perturbed dataset~$\tilde{\mathcal{D}}$.\footnote{We generate~20 perturbed sentences per source via random word dropping with probability of~0.1 and permutation with maximum distance of~3.}

Table~\ref{tab:mutual_info} shows the normalized mutual information~$\tilde{\mi{}} - \log|\mathcal{D}|$ where~$\tilde{\mi{}}$ denotes the estimated mutual information.
It shows that, when training with the combination of supervised and unsupervised objectives, the normalized mutual information is within a small range between~$(-2.6 \times 10^{-4}, -1.4 \times 10^{-4})$ and is lower than the maximum normalized mutual information~$\log|\tilde{\mathcal{D}}| - \log|\mathcal{D}| \approx 3.0$ by a large margin. Thus, the mutual information can be bounded by appropriate values of~$\mi{min}$ and~$\mi{max}$ to satisfy the constraint.
In addition, these results confirm that updating the inference model using policy gradient in Dual Learning does not effectively increase model's mutual information.

\begin{table}
\centering
\scalebox{1}{
\begin{tabular}{lcc}
\toprule
& {tr-en} & {en-tr} \\
\midrule
{baseline} & {-2.47} & {-2.28} \\
{epoch-level \abr{ibt}-1} & {-2.57} & {-2.23} \\
{epoch-level \abr{ibt}-2} & {-2.18} & {-2.30} \\
{epoch-level \abr{ibt}-3} & {-2.32} & {-2.42} \\
{batch-level \abr{ibt}} & {-1.51} & {-1.82} \\
{dual learning w/ $\alpha_{LM} = 0$} & {-1.50} & {-1.80} \\
{dual learning w/ $\alpha_{LM} = 0.5$} & {-1.44} & {-1.77} \\
\bottomrule
\end{tabular}}
\caption{Results on estimated mutual information~$\tilde{\mi{}}$ in the low-resource setting. We report the normalized scores~$\tilde{\mi{}} - \log|\mathcal{D}|$ (on the scale of~$10^{-4}$) averaged over the two test sets. The range of normalized scores should be~$[-\log|\mathcal{D}|, \log\frac{|\tilde{\mathcal{D}}|}{|\mathcal{D}|}] = [-8.0, 3.0]$.}
\label{tab:mutual_info}
\end{table}

\section{Summary of Contributions}

We contribute theoretical and empirical results that improve our understanding of the connection between two seemingly distant semi-supervised training strategies for NMT: Iterative Back-Translation (IBT) and Dual Learning.

\looseness=-1
On the theory side, we define a dual reconstruction objective which unifies semi-supervised NMT techniques that exploit source and target monolingual text. We prove that optimizing this objective leads to the same global optimum as the intractable marginal likelihood objective, where the model's marginal distribution coincides with the prior language distribution while also maximizing the model's mutual information between source and target. IBT approximates this objective more closely than Dual Learning, despite the more complex objective and update strategies used in the latter.

We present a systematic empirical comparison of Back-Translation, IBT, and Dual Learning on six tasks spanning high-resource, low-resource, and cross-domain settings. 
Results support the theory that the LM loss and policy gradient estimation are unnecessary in Dual Learning, and show that IBT achieves better translation quality than Dual Learning. 
Analysis confirms that the mutual information constraint required to reach an interesting dual reconstruction optimum is satisfied in practice.

These findings lead us to recommend batch-level IBT to quickly boost model performance at early training stages and epoch-level IBT to further improve quality. Our theory also suggests future directions for improving unsupervised MT via more effective methods to maximize the model’s mutual information between source and target, and the potential of applying our dual reconstruction objective to other sequence-to-sequence tasks.

\section*{Acknowledgments}
We thank the anonymous reviewers, Jordan Boyd-Graber, Hal Daum\'e III, Naomi Feldman, Shi Feng, Pranav Goel, Alexander Miserlis Hoyle, Michelle Yuan, Mozhi Zhang, and the CLIP lab at UMD for helpful comments. This research is supported in part by an Amazon Web Services Machine Learning Research Award and by the Office of the Director of National Intelligence (ODNI), Intelligence Advanced Research Projects Activity (IARPA), via contract \#FA8650-17-C-9117. The views and conclusions contained herein are those of the authors and should not be interpreted as necessarily representing the official policies, either expressed or implied, of ODNI, IARPA, or the U.S. Government. The U.S. Government is authorized to reproduce and distribute reprints for governmental purposes notwithstanding any copyright annotation therein.

\bibliography{anthology,emnlp2020}
\bibliographystyle{acl_natbib}

\clearpage
\appendix
\section{Proof}
\label{sec:appendix:proof}

\subsection{Proof for Proposition 1}
\label{sec:appendix:proof_1}
\paragraph{Proposition 1.}
\label{appendix:prop}
Given prior distributions~$q(\boldsymbol{x})$ and~$p(\boldsymbol{y})$ over~$\boldsymbol{x} \in \Sigma_x$ and~$\boldsymbol{y} \in \Sigma_y$, if parameterized probability models~$p_\theta$ and~$q_\phi$ have enough capacity under the constraint that:
\begin{equation*}
\begin{split}
    & 0 \le \mi{min} \le \mi{p_\theta}, \mi{q_\phi} \le \mi{max} \le \max_{p \in P_{XY}} \mi{p}(\boldsymbol{x}; \boldsymbol{y}) \\
\end{split}
\end{equation*}
where~$\mi{min}$ and~$\mi{max}$ are pre-defined constant values between zero and the maximum mutual information between~$\boldsymbol{x}$ and~$\boldsymbol{y}$ given any joint distribution~$p(\boldsymbol{x}, \boldsymbol{y}) \in P_{XY}$ whose marginals satisfy~$\sum_{\boldsymbol{x}} p(\boldsymbol{x}, \boldsymbol{y}) = p(\boldsymbol{y})$ and~$\sum_{\boldsymbol{y}} p(\boldsymbol{x}, \boldsymbol{y}) = q(\boldsymbol{x})$.
Then, the dual reconstruction objective is upper-bounded by~$\mathcal{J}_{dual}(\theta, \phi) \le 2\mi{max} - \h{}{q(\boldsymbol{x})} - \h{}{p(\boldsymbol{y})}$, and the upper bound is achieved iff
\begin{equation}
\label{appendix:eq:optcrit}
\begin{split}
    & \mi{q_{\phi}} = \mi{max} \\
    & \mi{p_{\theta}} = \mi{max} \\
    & p_{\theta}(\boldsymbol{y}\, | \,\boldsymbol{x})
        = \frac{q_{\phi}(\boldsymbol{x}\, | \,\boldsymbol{y})}{q_{\phi}(\boldsymbol{x})} p(\boldsymbol{y}) \\
    & q_{\phi}(\boldsymbol{x}\, | \,\boldsymbol{y})
        = \frac{p_{\theta}(\boldsymbol{y}\, | \,\boldsymbol{x})}{p_{\theta}(\boldsymbol{y})} q(\boldsymbol{x}) \\
\end{split}
\end{equation}

\begin{proof}
First we prove that~$\mathcal{J}_1(\theta, \phi) \le \mi{max} - \h{}{p(\boldsymbol{y})}$, and the upper bound is achieved iff
\begin{equation*}
\begin{split}
    & \mi{q_{\phi}} = \mi{max} \\
    & p_{\theta}(\boldsymbol{y}\, | \,\boldsymbol{x})
        = \frac{q_{\phi}(\boldsymbol{x}\, | \,\boldsymbol{y})}{q_{\phi}(\boldsymbol{x})} p(\boldsymbol{y})
\end{split}
\end{equation*}
where~$\h{}{p(\boldsymbol{y})}$ is the entropy of the prior distribution~$p(\boldsymbol{y})$.

To show this, we denote the posterior distribution~$Q(\boldsymbol{y}\, | \,\boldsymbol{x}) = \frac{q_{\phi}(\boldsymbol{x}\, | \,\boldsymbol{y})}{q_{\phi}(\boldsymbol{x})} p(\boldsymbol{y})$, and rewrite~$\mathcal{J}_1$:
\begin{equation*}
\begin{split}
    \mathcal{J}_1
    =& \e{\boldsymbol{y} \sim p(\boldsymbol{y})} { \e{\boldsymbol{x} \sim q_\phi(\boldsymbol{x}\, | \,\boldsymbol{y})}{\log p_\theta(\boldsymbol{y}\, | \,\boldsymbol{x})} } \\
    =& \e{\boldsymbol{y} \sim p(\boldsymbol{y})} { \e{\boldsymbol{x} \sim q_\phi(\boldsymbol{x}\, | \,\boldsymbol{y})}{\log Q(\boldsymbol{y}\, | \,\boldsymbol{x})} } \\
    &+ \e{\boldsymbol{y} \sim p(\boldsymbol{y})} { \e{\boldsymbol{x} \sim q_\phi(\boldsymbol{x}\, | \,\boldsymbol{y})} {log\frac{p_\theta(\boldsymbol{y}\, | \,\boldsymbol{x})}{Q(\boldsymbol{y}\, | \,\boldsymbol{x})}} } \\
    =& \e{\boldsymbol{y} \sim p(\boldsymbol{y})} { \e{\boldsymbol{x} \sim q_\phi(\boldsymbol{x}\, | \,\boldsymbol{y})}{\log Q(\boldsymbol{y}\, | \,\boldsymbol{x})} } \\
    &+ \e{\boldsymbol{y} \sim p(\boldsymbol{y})} { \e{\boldsymbol{x} \sim q_\phi(\boldsymbol{x}\, | \,\boldsymbol{y})} {log\frac{p_\theta(\boldsymbol{y}\, | \,\boldsymbol{x}) q_\phi(\boldsymbol{x})}{q_\phi(\boldsymbol{x}\, | \,\boldsymbol{y}) p(\boldsymbol{y})}} } \\
    =& \e{\boldsymbol{y} \sim p(\boldsymbol{y})} { \e{\boldsymbol{x} \sim q_\phi(\boldsymbol{x}\, | \,\boldsymbol{y})}{\log Q(\boldsymbol{y}\, | \,\boldsymbol{x})} } \\
    &- \kld{q_\phi(\boldsymbol{x}\, | \,\boldsymbol{y}) p(\boldsymbol{y})}{p_\theta(\boldsymbol{y}\, | \,\boldsymbol{x}) q_\phi(\boldsymbol{x})} \\
    =& \e{\boldsymbol{y} \sim p(\boldsymbol{y})} { \e{\boldsymbol{x} \sim q_\phi(\boldsymbol{x}\, | \,\boldsymbol{y})}{\log \frac{q_\phi(\boldsymbol{x}\, | \,\boldsymbol{y})}{q_\phi(\boldsymbol{x})} } } \\
    &+ \e{\boldsymbol{y} \sim p(\boldsymbol{y})} { \e{\boldsymbol{x} \sim q_\phi(\boldsymbol{x}\, | \,\boldsymbol{y})}{\log p(\boldsymbol{y})} } \\
    &- \kld{q_\phi(\boldsymbol{x}\, | \,\boldsymbol{y}) p(\boldsymbol{y})}{p_\theta(\boldsymbol{y}\, | \,\boldsymbol{x}) q_\phi(\boldsymbol{x})} \\
    =& \mi{q_\phi} - \h{}{p(\boldsymbol{y})} \\
    &- \kld{q_\phi(\boldsymbol{x}\, | \,\boldsymbol{y}) p(\boldsymbol{y})}{p_\theta(\boldsymbol{y}\, | \,\boldsymbol{x}) q_\phi(\boldsymbol{x})} \\
\end{split}
\end{equation*}
Since the KL divergence between two distributions is always non-negative and is zero iff they are equal, we have
\begin{equation*}
\begin{split}
    \mathcal{J}_1(\theta, \phi)
    \le \mi{q_\phi} - \h{}{p(\boldsymbol{y})}
    \le \mi{max} - \h{}{p(\boldsymbol{y})}
\end{split}
\end{equation*}
and~$\mathcal{J}_1(\theta, \phi) = \mi{max} - \h{}{p(\boldsymbol{y})}$ iff
\begin{equation*}
\begin{split}
    & \mi{q_{\phi}} = \mi{max} \\
    & \kld{q_{\phi}(\boldsymbol{x}\, | \,\boldsymbol{y}) p(\boldsymbol{y})}{p_{\theta}(\boldsymbol{y}\, | \,\boldsymbol{x}) q_{\phi}(\boldsymbol{x})} = 0 \\
\end{split}
\end{equation*}
The second equality holds iff
\begin{equation*}
    p_{\theta}(\boldsymbol{y}\, | \,\boldsymbol{x})
    = \frac{q_{\phi}(\boldsymbol{x}\, | \,\boldsymbol{y})}{q_{\phi}(\boldsymbol{x})} p(\boldsymbol{y})
\end{equation*}

Similarly, we can prove that~$\mathcal{J}_2(\theta, \phi) \le \mi{max} - \h{}{q(\boldsymbol{x})}$, and the upper bound is achieved iff
\begin{equation*}
\begin{split}
    & \mi{p_{\theta}} = \mi{max} \\
    & q_{\phi}(\boldsymbol{x}\, | \,\boldsymbol{y})
        = \frac{p_{\theta}(\boldsymbol{y}\, | \,\boldsymbol{x})}{p_{\theta}(\boldsymbol{y})} q(\boldsymbol{x}) \\
\end{split}
\end{equation*}
thus~$\mathcal{J}_{dual}(\theta, \phi) \le 2\mi{max} - \h{}{q(\boldsymbol{x})} - \h{}{p(\boldsymbol{y})}$ and the upper bound is achieved iff~$\theta$ and~$\phi$ satisfy~\cref{appendix:eq:optcrit}, concluding the proof.
\end{proof}

\subsection{Proof for Proposition 2}
\label{sec:appendix:proof_2}
\paragraph{Proposition 2.}
\label{appendix:prop:construct}
Given distributions~$q(\boldsymbol{x})$ and~$p(\boldsymbol{y})$ over~$\boldsymbol{x} \in \Sigma_x$ and~$\boldsymbol{y} \in \Sigma_y$, if parameterized probability models~$p_\theta$ and~$q_\phi$ have enough capacity under the constraint that:
\begin{equation}
\label{appendix:eq:constraint}
\begin{split}
    & 0 \le \mi{min} \le \mi{p_\theta}, \mi{q_\phi} \le \mi{max} \le \max_{p \in P_{XY}} \mi{p}(\boldsymbol{x}; \boldsymbol{y}) \\
\end{split}
\end{equation}
where~$\mi{min}$ and~$\mi{max}$ are pre-defined constant values between zero and the maximum mutual information between~$\boldsymbol{x}$ and~$\boldsymbol{y}$ given any joint distribution~$p(\boldsymbol{x}, \boldsymbol{y}) \in P_{XY}$ whose marginals satisfy~$\sum_{\boldsymbol{x}} p(\boldsymbol{x}, \boldsymbol{y}) = p(\boldsymbol{y})$ and~$\sum_{\boldsymbol{y}} p(\boldsymbol{x}, \boldsymbol{y}) = q(\boldsymbol{x})$.
Then there exist~$\theta^*$ and~$\phi^*$ such that:
\begin{equation}
\label{appendix:eq:crit}
\begin{split}
    & \mi{q_{\phi^*}} = \mi{p_{\theta^*}} = \mi{max} \\
    & p_{\theta^*}(\boldsymbol{y}\, | \,\boldsymbol{x})
        = \frac{q_{\phi^*}(\boldsymbol{x}\, | \,\boldsymbol{y})}{q_{\phi^*}(\boldsymbol{x})} p(\boldsymbol{y}) \\
    & q_{\phi^*}(\boldsymbol{x}\, | \,\boldsymbol{y})
        = \frac{p_{\theta^*}(\boldsymbol{y}\, | \,\boldsymbol{x})}{p_{\theta^*}(\boldsymbol{y})} q(\boldsymbol{x}) \\
\end{split}
\end{equation}

\begin{proof}
Since~$\mi{max}$ satisfies
\begin{equation*}
0 = \min_{p \in P_{XY}} \mi{p}(\boldsymbol{x}; \boldsymbol{y}) \le \mi{max} \le \max_{p \in P_{XY}} \mi{p}(\boldsymbol{x}; \boldsymbol{y})
\end{equation*}
there exists a joint distribution~$p^*(\boldsymbol{x}, \boldsymbol{y}) \in P_{XY}$ such that
\begin{equation*}
    \mi{p^*}(\boldsymbol{x}; \boldsymbol{y}) = \mi{max}
\end{equation*}

As models~$p_\theta$ and~$q_\phi$ have enough capacity under the constraint in~\cref{appendix:eq:constraint}, there exist~$\theta^*$ and~$\phi^*$ such that~$\forall x \in \Sigma_x, \forall y \in \Sigma_y$
\begin{equation*}
\begin{split}
    & p_{\theta^*}(\boldsymbol{y}\, | \,\boldsymbol{x}) = \frac{p^*(\boldsymbol{x}, \boldsymbol{y})}{q(\boldsymbol{x})} \\
    & q_{\phi^*}(\boldsymbol{x}\, | \,\boldsymbol{y}) = \frac{p^*(\boldsymbol{x}, \boldsymbol{y})}{p(\boldsymbol{y})}
\end{split}
\end{equation*}
thus
\begin{equation*}
\begin{split}
    & p_{\theta^*}(\boldsymbol{y}) = \sum_{\boldsymbol{x}} p_{\theta^*}(\boldsymbol{y}\, | \,\boldsymbol{x}) q(\boldsymbol{x}) = \sum_{\boldsymbol{x}} p^*(\boldsymbol{x}, \boldsymbol{y}) = p(\boldsymbol{y}) \\
    & q_{\phi^*}(\boldsymbol{x}) = \sum_{\boldsymbol{y}} q_{\phi^*}(\boldsymbol{x}\, | \,\boldsymbol{y}) p(\boldsymbol{y}) = \sum_{\boldsymbol{y}} p^*(\boldsymbol{x}, \boldsymbol{y}) = q(\boldsymbol{x})
\end{split}    
\end{equation*}
and thus
\begin{equation*}
\begin{split}
    \mi{p_{\theta^*}}
    &= \e{\boldsymbol{x} \sim q(\boldsymbol{x})} { \e{\boldsymbol{y} \sim p_{\theta^*}(\boldsymbol{y}\, | \,\boldsymbol{x})}{\log \frac{p_{\theta^*}(\boldsymbol{y}\, | \,\boldsymbol{x})}{p_{\theta^*}(\boldsymbol{y})} } } \\
    &= \e{\boldsymbol{x}, \boldsymbol{y} \sim p^*(\boldsymbol{x}, \boldsymbol{y})}{\log \frac{p^*(\boldsymbol{x}, \boldsymbol{y})}{q(\boldsymbol{x}) p(\boldsymbol{y}) }} \\
    &= \mi{p^*}(\boldsymbol{x}; \boldsymbol{y}) \\
    &= \mi{max} \\
\end{split}
\end{equation*}

\begin{equation*}
\begin{split}
    \mi{q_{\phi^*}}
    &= \e{\boldsymbol{y} \sim p(\boldsymbol{y})} { \e{\boldsymbol{x} \sim q_{\phi^*}(\boldsymbol{x}\, | \,\boldsymbol{y})}{\log \frac{q_{\phi^*}(\boldsymbol{x}\, | \,\boldsymbol{y})}{q_{\phi^*}(\boldsymbol{x})} } } \\
    &= \e{\boldsymbol{x}, \boldsymbol{y} \sim p^*(\boldsymbol{x}, \boldsymbol{y})}{\log \frac{p^*(\boldsymbol{x}, \boldsymbol{y})}{q(\boldsymbol{x}) p(\boldsymbol{y}) }} \\
    &= \mi{p^*}(\boldsymbol{x}; \boldsymbol{y}) \\
    &= \mi{max} \\
\end{split}
\end{equation*}
and
\begin{equation*}
    \frac{q_{\phi^*}(\boldsymbol{x}\, | \,\boldsymbol{y})}{q_{\phi^*}(\boldsymbol{x})} p(\boldsymbol{y})
    = \frac{p^*(\boldsymbol{x}, \boldsymbol{y})}{p(\boldsymbol{y})} \frac{p(\boldsymbol{y})}{q(\boldsymbol{x})}
    = p_{\theta^*}(\boldsymbol{y}\, | \,\boldsymbol{x})
\end{equation*}

\begin{equation*}
    \frac{p_{\theta^*}(\boldsymbol{y}\, | \,\boldsymbol{x})}{p_{\theta^*}(\boldsymbol{y})} q(\boldsymbol{x})
    = \frac{p^*(\boldsymbol{x}, \boldsymbol{y})}{q(\boldsymbol{x})} \frac{q(\boldsymbol{x})}{p(\boldsymbol{y})}
    = q_{\phi^*}(\boldsymbol{x}\, | \,\boldsymbol{y})
\end{equation*}
concluding the proof.
\end{proof}

\subsection{Proof for Theorem 1}
\label{sec:appendix:theorem}
\paragraph{Theorem 1.}
\label{appendix:theorem}
Given prior distributions~$q(\boldsymbol{x})$ and~$p(\boldsymbol{y})$ over~$\boldsymbol{x} \in \Sigma_x$ and~$\boldsymbol{y} \in \Sigma_y$, if parameterized probability models~$p_\theta$ and~$q_\phi$ have enough capacity under the constraint that:
\begin{equation*}
\begin{split}
    & 0 \le \mi{min} \le \mi{p_\theta}, \mi{q_\phi} \le \mi{max} \le \max_{p \in P_{XY}} \mi{p}(\boldsymbol{x}; \boldsymbol{y}) \\
\end{split}
\end{equation*}
where~$\mi{min}$ and~$\mi{max}$ are pre-defined constant values between zero and the maximum mutual information between~$\boldsymbol{x}$ and~$\boldsymbol{y}$ given any joint distribution~$p(\boldsymbol{x}, \boldsymbol{y}) \in P_{XY}$ whose marginals satisfy~$\sum_{\boldsymbol{x}} p(\boldsymbol{x}, \boldsymbol{y}) = p(\boldsymbol{y})$ and~$\sum_{\boldsymbol{y}} p(\boldsymbol{x}, \boldsymbol{y}) = q(\boldsymbol{x})$.
Let $\theta^*, \phi^*$ be the global optimum of the dual reconstruction objective~$\max_{\theta, \phi} \mathcal{J}_{dual}(\theta, \phi)$, then~$q_{\phi^*}(\boldsymbol{x}) = q(\boldsymbol{x})$,~$p_{\theta^*}(\boldsymbol{y}) = p(\boldsymbol{y})$, and~$\mi{q_{\phi^*}} = \mi{p_{\theta^*}} = \mi{max}$.

\begin{proof}
Suppose models~$p_\theta$ and~$q_\phi$ have enough capacity under the constraint that:
\begin{equation*}
\begin{split}
    & 0 \le \mi{min} \le \mi{p_\theta}, \mi{q_\phi} \le \mi{max} \le \max_{p \in P_{XY}} \mi{p}(\boldsymbol{x}; \boldsymbol{y}) \\
\end{split}
\end{equation*}
then based on~\cref{prop},~$\mathcal{J}_{dual}(\theta, \phi) \le 2\mi{max} - \h{}{q(\boldsymbol{x})} - \h{}{p(\boldsymbol{y})}$, and the upper bound is achieved iff the optimal criteria~\cref{appendix:eq:optcrit} hold.
And based on~\cref{prop:construct}, there exists a solution~$p_{\theta^*}$ an $q_{\phi^*}$ for the criteria~\cref{appendix:eq:optcrit}.
Thus $$\max_{\theta, \phi} \mathcal{J}_{dual}(\theta, \phi) =2\mi{max} - \h{}{q(\boldsymbol{x})} - \h{}{p(\boldsymbol{y})}$$

Based on the first equation in~\cref{appendix:eq:optcrit}, we have:
\begin{equation}
    \mi{q_{\phi^*}} = \mi{p_{\theta^*}} = \mi{max}
\end{equation}
And multiply the last two equations, we have:
\begin{equation}
\label{appendix:eq:product}
\begin{split}
    p(\boldsymbol{y}) q(\boldsymbol{x}) = q_{\phi^*}(\boldsymbol{x}) p_{\theta^*}(\boldsymbol{y})
\end{split}
\end{equation}
Given Lemma~1, we have~$q_{\phi^*}(\boldsymbol{x}) = q(\boldsymbol{x})$ and~$p_{\theta^*}(\boldsymbol{y}) = p(\boldsymbol{y})$, concluding the proof.
\end{proof}

\begin{table*}[ht]
\centering
\begin{tabular}{lrrrr}
\toprule
\textbf{Task} & {NMT.xx-en} & {NMT.en-xx} & {LM.xx} & {LM.en} \\\hline
{\bf low-resource} & 22.12 &	30.92 &	121.13 &	78.96 \\
{\bf high-resource} & 6.72 &	6.25 &	78.32 &	74.00 \\
{\bf cross-domain} & 8.17 &	7.35 &	102.43 &	92.70 \\\hline
\end{tabular}
\caption{Validation perplexity of the NMT and LM models. We denote English as \textit{en} and the other language as \textit{xx}.}
\label{tab:valid_ppl}
\end{table*}

\begin{table*}[ht]
\centering
\begin{tabular}{lrrrr}
\toprule
\textbf{Task} & {NMT.xx-en} & {NMT.en-xx} & {LM.xx} & {LM.en} \\\hline
{\bf low-resource} & 92811565 & 92811565 & 27311123	& 19247448 \\
{\bf high-resource} & 98346302 &	98346302 &	32165523 &	24095698 \\
{\bf cross-domain} & 98346302 &	98346302 &	24806023 &	18768773 \\\hline
\end{tabular}
\caption{Number of model parameters. We denote English as \textit{en} and the other language as \textit{xx}.}
\label{tab:model_size}
\end{table*}

\paragraph{Lemma 1.}
\label{lemma:eq}
\looseness=-1
Let~$p(x)$ and~$p'(x)$ be two discrete probability functions over random variable $x \in \Sigma_x$, and $q(y)$ and~$q'(y)$ be two discrete probability functions over random variable~$y \in \Sigma_y$.\\
If~$\forall x \in \Sigma_x, \forall y \in \Sigma_y,\, p'(x)q'(y) = p(x)q(y)$, then~$\forall x \in \Sigma_x, \forall y \in \Sigma_y$,~$p'(x) = p(x)$ and~$q'(y) = q(y)$.

\begin{proof}
Let~$x_0 \in \Sigma_x$ such that~$p(x_0) \neq 0$ and~$p'(x_0) \neq 0$, and~$y_0 \in \Sigma_y$ such that~$q(y_0) \neq 0$ and~$q'(y_0) \neq 0$.

Since
\begin{equation}
    p'(x_0)q'(y_0) = p(x_0)q(y_0)
\end{equation}
and for any~$x \in \Sigma_x$
\begin{equation}
    p'(x)q'(y_0) = p(x)q(y_0)
\end{equation}
we have
\begin{equation}
    \frac{p'(x)}{p'(x_0)} = \frac{p(x)}{p(x_0)}
\end{equation}.

\looseness=-1
Given $\sum_{x \in \Sigma_x} p'(x) = 1$, $\sum_{x \in \Sigma_x} p(x) = 1$, and since
\begin{equation}
    \sum_{x \in \Sigma_x} p'(x) = \frac{p'(x_0)}{p(x_0)} \sum_{x \in \Sigma_x} p(x)
\end{equation}
we have~$p'(x_0) = p(x_0)$.

Thus for any~$x \in \Sigma_x$,~$p'(x) = \frac{p'(x_0)}{p(x_0)} p(x) = p(x)$. For any~$y \in \Sigma_y$,~$q'(y) = \frac{p(x)}{p'(x)} q(y) = q(y)$, concluding the proof.
\end{proof}

\section{Experimental Setup}
\label{sec:appendix:exp}
\subsection{Tasks and Data}
We evaluate on six translation tasks including  German$\leftrightarrow$English~(de-en),\footnote{We exclude Rapid and ParaCrawl corpora as they are noisy and thus require data filtering~\cite{Morishita2018}.} Turkish$\leftrightarrow$English~(tr-en) from WMT18~\citep{BojarFFGHHKM18},\footnote{\url{http://www.statmt.org/wmt18/translation-task.html}} and a cross-domain task which tests de$\leftrightarrow$en models trained on WMT data on the TED test sets from IWSLT17~\citep{IWSLT17}.\footnote{\url{https://wit3.fbk.eu/mt.php?release=2017-01-ted-test}}

\subsection{Model and Training Configuration}
We adopt the base Transformer model~\citep{Vaswani2017} with~$d_{\text{model}}=512$,~$d_{\text{hidden}}=2048$,~$n_{\text{heads}}=8$,~$n_{\text{layers}}=6$, and $p_\text{drop}=0.1$. We tie the source and target embeddings with the output layer weights~\citep{PressW17,NguyenC18}.

We use the Adam optimizer~\citep{KingmaB15} with a batch size of~$32$ sentences and checkpoint the model every~$2500$ updates. Training hyperparameters and stopping criteria are constant across all comparable experimental conditions. Initial learning rates for pre-training and fine-tuning are respectively set to~$10^{-4}$ and~$2 \times 10^{-5}$. We decay the learning rate by~$30\%$ and reload the best model after~$3$ checkpoints without improvement. We apply early stopping after repeating this process for~$5$ times. We adopt the same learning rate decay and stopping criteria during fine-tuning. For batch-level IBT and Dual Learning, we check whether both models improve validation perplexity. For epoch-level IBT, we run for~$3$ iterations.

The LMs in Dual Learning are RNNs~\citep{MikolovKBCK10} with~$512$ hidden units, embeddings of size~$512$, and dropout of~$0.2$ to hidden states. We tie the input embeddings with the output layer weights. We clip the gradients at a threshold of~$5$. We train them similarly to NMT models, except setting the batch size to~$64$ sentences and the initial learning rate to~$0.001$. We decay the learning rate by~$50\%$ and reload the best model after~$5$ checkpoints without validation perplexity improvement and apply early stopping after repeating the process for~$5$ times. We report the validation perplexity of the NMT and LM models in Table~\ref{tab:valid_ppl}, and the model sizes in Table~\ref{tab:model_size}. All experiments are performed on a single NVIDIA GeForce GTX 1080 Ti GPU.
\end{document}